\documentclass{article}

\usepackage{iclr2023_conference,times}

\input{preamble}
\SetCommentSty{textnormal}
\SetKwInOut{Output}{output}
\SetKwInOut{Input}{input}
\DontPrintSemicolon
\SetNoFillComment
\SetAlgoNoLine

\newcommand{\defemph}[1]{\emph{#1}}

\newcommand{\feddc}{\textsc{FedDC}\xspace}
\newcommand{\fedavg}{\textsc{FedAvg}\xspace}
\newcommand{\fedyogi}{\textsc{FedYogi}\xspace}
\newcommand{\fedadam}{\textsc{FedAdam}\xspace}
\newcommand{\fedadagrad}{\textsc{FedAdagrad}\xspace}
\newcommand{\fedprox}{\textsc{FedProx}\xspace}

\newcommand{\central}{\textsc{central}\xspace}

\newtheorem{assumption}[theorem]{Assumption}

\newtheorem*{proposition*}{Proposition}
\newtheorem*{lemma*}{Lemma}
\newtheorem*{corollary*}{Corollary}

\newcommand{\algo}{\mathcal{A}}
\newcommand{\risk}{\varepsilon}
\newcommand{\model}{h}
\newcommand{\modelspace}{\mathcal{H}}
\newcommand{\aggmodel}{\overline{\model}}
\newcommand{\loss}{\ell}
\newcommand{\samplesize}{n}
\newcommand{\radonpoint}{\mathfrak{r}}

\DeclareMathOperator{\agg}{agg}

\newcommand{\Exp}[2]{\mathop{{}\mathbb{E}_{#1}} \Big[ #2 \Big] }

\newcommand{\prob}[1]{\mathop{{}\mathbb{P}} \left(  #1 \right ) }

\newcommand{\Dcal}{\mathcal{D}}
\newcommand{\RR}{\mathbb{R}}
\newcommand{\NN}{\mathbb{N}}

\newcommand{\Xcal}{\mathcal{X}}
\newcommand{\Ycal}{\mathcal{Y}}

\newcommand{\ourmaintitle}{Federated Learning from Small Datasets}

\newcommand{\ourtitle}{\ourmaintitle}
\newcommand{\ourmethod}{\feddc}
\newcommand{\oururl}{https://github.com/kampmichael/FedDC}
\usepackage[pdftitle={\ourtitle}, pdfauthor={Michael Kamp, Jonas Fischer, Jilles Vreeken}, breaklinks,colorlinks,pagebackref, hyperfootnotes = false,citecolor=darkblue, urlcolor=darkblue,linkcolor=darkblue]{hyperref}

\graphicspath{ {./figs/} }

\ifpdf
\usetikzlibrary{external}
\fi
\usetikzlibrary{patterns}

\usepackage{eqparbox}

\title{\ourtitle}

\author{%
  Michael Kamp \\
  Institute for AI in medicine (IKIM)\\
  University Hospital Essen, Essen Germany, and\\
  Ruhr-University Bochum, Bochum Germany, and\\
  Monash University, Melbourne, Australia\\
  \texttt{michael.kamp@uk-essen.de}\\
  \And
  Jonas Fischer\\
  Harvard T.H. Chan School of Public Health\\
  Department of Biostatistics\\
  Boston, MA, United States\\
  \texttt{jfischer@hsph.harvard.edu}\\
  \AND
  Jilles Vreeken\\
  CISPA Helmholtz Center for Information Security\\
  Saarbr\"ucken, Germany\\
  \texttt{vreeken@cispa.de}\\
}

\iclrfinalcopy

\begin{document}
	
	\maketitle

	\begin{abstract}
		Federated learning allows multiple parties to collaboratively train a joint model without having to share any local data. It enables applications of machine learning in settings where data is inherently distributed and undisclosable, such as in the medical domain. Joint training is usually achieved by aggregating local models. When local datasets are small, locally trained models can vary greatly from a globally good model. Bad local models can arbitrarily deteriorate the aggregate model quality, causing federating learning to fail in these settings. 
We propose a novel approach that avoids this problem by interleaving model \emph{aggregation} and \emph{permutation} steps. During a permutation step we redistribute local models across clients through the server, while preserving data privacy, to allow each local model to train on a daisy chain of local datasets. This enables successful training in data-sparse domains. Combined with model aggregation, this approach enables effective learning even if the local datasets are extremely small, while retaining the privacy benefits of federated learning.
	\end{abstract}

	\section{Introduction}
How can we learn \emph{high quality} models when data is \emph{inherently distributed} across sites and cannot be shared or pooled?
In federated learning, the solution is to iteratively train models locally at each site and share these models with the server to be aggregated to a global model. 
As only models are shared, data usually remains undisclosed.
This process, however, requires sufficient data to be available at each site in order for the locally trained models to achieve a minimum quality---even a single bad model can render aggregation arbitrarily bad~\citep{shamir2014distributed}.
In many relevant applications this requirement is not met: In healthcare settings we often have as little as a few dozens of samples~\citep{granlund2020hyperpolarized, su2021comprehensive, painter2020angiosarcoma}. Also in domains where deep learning is generally regarded as highly successful, such as natural language processing and object detection, applications often suffer from a lack of data~\citep{liu2020survey, kang2019few}.

To tackle this problem, we propose a new building block called \textit{daisy-chaining} for federated learning in which models are trained on one local dataset after another, much like a daisy chain. In a nutshell, at each client a model is trained locally, sent to the server, and then---instead of aggregating local models---sent to a random other client as is (see Fig.~\ref{fig:overview}). This way, each local model is exposed to a daisy chain of clients and their local datasets.
This allows us to learn from small, distributed datasets simply by consecutively training the model with the data available at each site.
Daisy-chaining alone, however, violates privacy, since a client can infer from a model upon the data of the client it received it from~\citep{shokri2017membership}. Moreover, performing daisy-chaining naively would lead to overfitting which can cause learning to diverge~\citep{haddadpour2019convergence}.
In this paper, we propose to combine daisy-chaining of local datasets with aggregation of models, both orchestrated by the server, and term this method \emph{federated daisy-chaining} (\ourmethod).
%

We show that our simple, yet effective approach maintains privacy of local datasets, while it provably converges and guarantees improvement of model quality in convex problems with a suitable aggregation method. 
Formally, we show convergence for \ourmethod on non-convex problems. We then show for convex problems that \ourmethod succeeds on small datasets where standard federated learning fails. For that, we analyze \ourmethod combined with aggregation via the Radon point from a PAC-learning perspective. We substantiate this theoretical analysis for convex problems by showing that \ourmethod in practice matches the accuracy of a model trained on the full data of the SUSY binary classification dataset with only $2$ samples per client, outperforming standard federated learning by a wide margin. 
For non-convex settings, we provide an extensive empirical evaluation, showing that \ourmethod outperforms naive daisy-chaining, vanilla federated learning \fedavg~\citep{mcmahan2017communication}, \fedprox~\citep{li2020federated}, \fedadagrad, \fedadam, and \fedyogi~\citep{reddi2020adaptive} on  low-sample CIFAR10~\citep{krizhevsky2009learning}, including non-iid settings, and, more importantly, on two real-world medical imaging datasets. Not only does \ourmethod provide a wide margin of improvement over existing federated methods
, but it comes close to the performance of a gold-standard (centralized) neural network of the same architecture trained on the pooled data.
To achieve that, it requires a small communication overhead compared to standard federated learning for the additional daisy-chaining rounds. 
As often found in healthcare, we consider a cross-SILO scenario where such small communication overhead is negligible.
Moreover we show that with equal communication, standard federated averaging still underperforms in our considered settings.

In summary, our contributions are (i) \ourmethod, a novel approach to federated learning from small datasets via a combination of model permutations across clients and aggregation, (ii) a formal proof of convergence for \ourmethod, (iii) a theoretical guarantee that \ourmethod improves models in terms of $\epsilon,\delta$-guarantees which standard federated learning can not, (iv) a discussion of the privacy aspects and mitigations suitable for \ourmethod, including an empirical evaluation of differentially private \ourmethod, and (v) an extensive set of experiments showing that \ourmethod substantially improves model quality on small datasets compared to standard federated learning approaches.


\begin{figure*}
    \centering
    \includegraphics[width=1.0\textwidth]{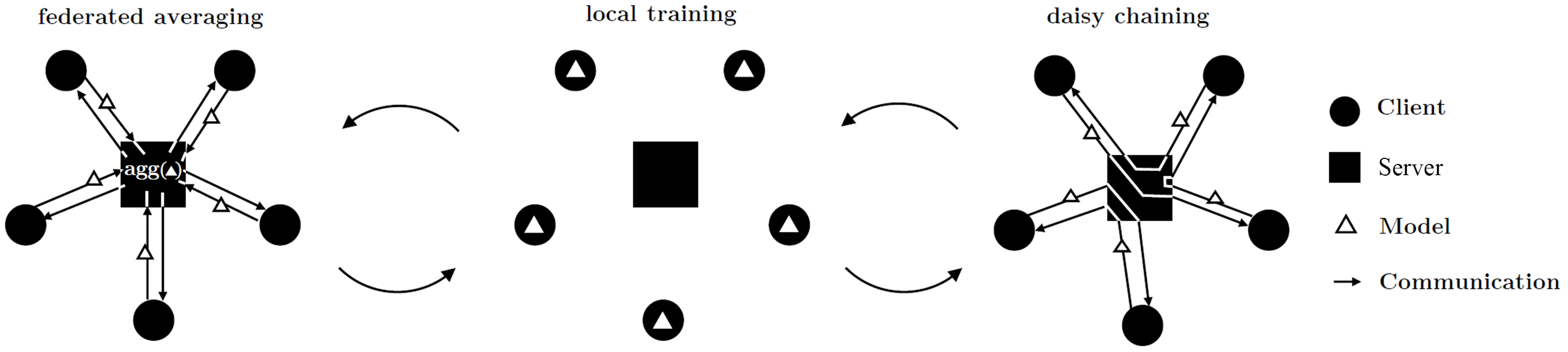}
    \caption{\textit{Federated learning settings.} A standard federated learning setting with training of local models at clients (middle) with aggregation phases where models are communicated to the server, aggregated, and sent back to each client (left).
    We propose to add daisy chaining (right), where local models are sent to the server and then redistributed to a random permutation of clients as is.}
    \label{fig:overview}
\end{figure*}
	\section{Related Work}
Learning from small datasets is a well studied problem in machine learning.
In the literature, we find both general solutions, such as using simpler models and transfer learning~\citep{torrey2010transfer}, and more specialized ones, such as data augmentation~\citep{ibrahim2021augmentation} and few-shot learning~\citep{vinyals2016matching, prabhu2019few}.
In our scenario overall data is abundant, but the problem is that data is distributed into small local datasets at each site, which we are not allowed to pool. \citet{hao2021towards} propose local data augmentation for federated learning, but their method requires a sufficient quality of the local model for augmentation which is the opposite of the scenario we are considering.
\citet{huang2021fl} provide generalization bounds for federated averaging via the NTK-framework, but requires one-layer infinite-width NNs and infinitesimal learning rates.

Federated learning and its variants have been shown to learn from incomplete local data sources, e.g., non-iid label distributions~\citep{li2020federated, wang2019federated} and differing feature distributions~\citep{li2020fedbn, reisizadeh2020robust}, but fail in case of large gradient diversity~\citep{haddadpour2019convergence} and strongly dissimilar label distribution~\citep{marfoq2021federated}. For small datasets, local empirical distributions may vary greatly from the global distribution: the difference of empirical to true distribution decreases exponentially with the sample size (e.g., according to the Dvoretzky–Kiefer–Wolfowitz inequality), but for small samples the difference can be substantial, in particular if the distribution differs from a Normal distribution~\citep{kwak2017central}. \citet{shamir2014distributed} have shown the adverse effect of bad local models on averaging, proving that even due to a single bad model averaging can be arbitrarily bad.

A different approach to dealing with biased local data is by learning personalized models at each client. Such 
personalized FL~\citep{li2021model} can reduce sample complexity, e.g., by using shared representations~\citep{pmlr-v139-collins21a} for client-specific models, e.g., in the medical domain~\citep{yang2021flop}, or by training sample-efficient personalized Bayesian methods~\citep{achituve2021personalized}. It is not applicable, however, to settings where you are not allowed to learn the biases or batch effects of local clients, e.g., in many medical applications where this would expose sensitive client information.
\citet{kiss2021migrating} propose a decentralized and communication-efficient variant of federated learning that migrates models over a decentralized network, storing incoming models locally at each client until sufficiently many models are collected on each client for an averaging step, similar to Gossip federated learing~\citep{jelasity2005gossip}. The variant without averaging is similar to simple daisy-chaining which we compare to in Section~\ref{sec:experiments}.
\ourmethod is compatible with any aggregation operator, including the Radon machine~\citep{kamp2017effective}, the geometric median~\citep{pillutla2022robust}, or neuron-clustering~\citep{yurochkin2019bayesian}, and can be straightforwardly combined with approaches to improve communication-efficiency, such as dynamic averaging~\citep{kamp2018efficient}, and model quantization~\citep{reisizadeh2020fedpaq}. We combine \ourmethod with averaging, the Radon machine, and FedProx~\citep{li2020federated} in Sec.~\ref{sec:experiments}.

	\section{Preliminaries}
\label{sec:prelim}
We assume iterative learning algorithms~\citep[cf. Chp. 2.1.4][]{kamp2019black} $\algo:\Xcal\times\Ycal\times\modelspace\rightarrow\modelspace$ that update a model $\model\in\modelspace$ using a dataset $D\subset\Xcal\times\Ycal$ from an input space $\Xcal$ and output space $\Ycal$, i.e., $\model_{t+1}=\algo(D,\model_t)$. 
Given a set of $m\in\NN$ clients with local datasets $D^1,\dots,D^m\subset\Xcal\times\Ycal$ drawn iid from a data distribution $\Dcal$ 
and a loss function $\loss:\Ycal\times\Ycal\rightarrow\RR$, the goal is to find a single model $\model^*\in\modelspace$ that minimizes the risk $\risk(\model)=\mathbb{E}_{(x,y)\sim\Dcal}[\loss(\model(x),y)]$.
In \defemph{centralized learning}, datasets are pooled as $D=\bigcup_{i\in [m]}D^i$ and $\algo$ is applied to $D$ until convergence. Note that applying $\algo$ on $D$ can be the application to any random subset, e.g., as in mini-batch training, and convergence is measured in terms of low training loss, small gradient, or small deviation from previous iterate. In standard \defemph{federated learning}~\citep{mcmahan2017communication}, $\algo$ is applied in parallel for $b\in\NN$ rounds on each client locally to produce local models $\model^1,\dots,\model^m$. These models are then centralized and aggregated using an aggregation operator $\agg:\modelspace^m\rightarrow\modelspace$, i.e., $\aggmodel = \agg(\model^1,\dots,\model^m)$. The aggregated model $\aggmodel$ is then redistributed to local clients which perform another $b$ rounds of training using $\aggmodel$ as a starting point. This is iterated until convergence of $\aggmodel$. When aggregating by averaging, this method is known as federated averaging (\fedavg). Next, we describe \ourmethod.
	\section{Federated Daisy-Chaining}
\label{sec:algo}
We propose federated daisy chaining as an extension to federated learning in a setup with $m$ clients and one designated sever.\!\footnote{This star-topology can be extended to hierarchical networks in a straightforward manner. Federated learning can also be performed in a decentralized network via gossip algorithms~\citep{jelasity2005gossip}.}
We provide the pseudocode of our approach as Algorithm~\ref{alg:feddc}.

\textbf{The client: } Each client trains its local model in each round on local data (line 4), and sends its model to the server every $b$ rounds for aggregation, where $b$ is the aggregation period, and every $d$ rounds for daisy chaining, where $d$ is the daisy-chaining period (line 6). This re-distribution of models results in each individual model conceptually following a daisy chain of clients, training on each local dataset. Such a daisy chain is interrupted by each aggregation round. 

\textbf{The server: }
Upon receiving models, in a daisy-chaining round (line 9) the server draws a random permutation $\pi$ of clients (line 10) and re-distributes the model of client $i$ to client $\pi(i)$ (line 11), while in an aggregation round (line 12), the server instead aggregates all local models and re-distributes the aggregate to all clients (line 13-14). 

\SetKwFor{local}{Locally}{do}{}
\SetKwFor{coord}{At server}{do}{}
\begin{algorithm}[t]
    \caption{Federated Daisy-Chaining \ourmethod}
    \label{alg:feddc}
    \KwIn{daisy-chaining period $d$, aggregation period $b$, learning algorithm $\algo$, aggregation operator $\agg$, $m$ clients with local datasets $D^1,\dots,D^m$, total number of rounds $T$}
    \KwOut{final model aggregate $\overline{h}_T$}
    initialize local models $h_0^1,\dots,h_0^m$\\
    \local{at client $i$ at time $t$}{
        sample $S$ from $D^i$\\
        $h_{t}^i\leftarrow\algo(S,h^i_{t-1})$\\
        \If{$t ~mod~ d = d-1$ or $t ~mod~ b = b-1$}{
             send $h_{t}^i$ to server\\
             receive new $h_{t}^i$ from server \tcp*[f]{receives either aggregate $\overline{h}_t$ or some $h_t^j$}\\
        }
    }
    \coord{at time $t$}{
        \If(\tcp*[f]{daisy chaining}){$t ~mod~ d = d-1$}{
            draw permutation $\pi$ of [1,m] at random\\
                for all $i\in [m]$ send model $h_t^i$ to client $\pi(i)$\\
        }
        \ElseIf(\tcp*[f]{aggregation}){$t ~mod~ b = b-1$}{
            $\overline{h}_t\leftarrow\agg(h_t^1,\dots,h_t^m)$\\
            send $\overline{h}_t$ to all clients 
        }
    }
\end{algorithm}

\textbf{Communication complexity: } 
Note that we consider cross-SILO settings, such as healthcare, were communication is not a bottleneck and, hence, restrict ourselves to a brief discussion in the interest of space.
Communication between clients and server happens in $O(\frac{T}{d} + \frac{T}{b})$ many rounds, where $T$ is the overall number of rounds. Since \ourmethod communicates every $d$th and $b$th round, 
the amount of communication rounds is similar to \fedavg with averaging period $b_{\mathit{FedAvg}}=\min\{d,b\}$. That is, \ourmethod increases communication over \fedavg by a constant factor depending on the setting of $b$ and $d$.
The amount of communication per communication round is linear in the number of clients and model size, similar to federated averaging.
We investigate the performance of \fedavg provided with the same communication capacity as \feddc in our experiments and in App.~\ref{app:commeff}.

	\section{Theoretical guarantees}
\label{sec:theory}
In this section, we formally show that \ourmethod converges for averaging. We, further, provide theoretical bounds on the model quality in convex settings, showing that \ourmethod has favorable generalization error in low sample settings compared to standard federated learning. 
More formally, we first show that under standard assumptions on the empirical risk, it follows from a result of~\citet{yu2019parallel} that \ourmethod converges when using averaging as aggregation and SGD for learning---a standard setting in, e.g., federated learning of neural networks. We provide all proofs in the appendix.

\begin{corollary}
Let the empirical risks $\mathcal{E}_{emp}^i(h)=\sum_{(x,y)\in D^i}\ell(h_i(x),y)$ at each client $i\in [m]$ be $L$-smooth with $\sigma^2$-bounded gradient variance and $G^2$-bounded second moments, then \ourmethod with averaging and SGD 
has a convergence rate of $\mathcal{O}(1/\sqrt{m T})$, where $T$ is the number of local updates.
\label{cor:convergence}
\end{corollary}

Since model quality in terms of generalization error does not necessarily depend on convergence of training~\citep{haddadpour2019convergence, kamp2018efficient}, we additionally analyze model quality in terms of probabilistic worst-case guarantees on the generalization error~\citep{shalev2014understanding}. %
The average of local models can yield as bad a generalization error as the worst local model, hence, using averaging as aggregation scheme in standard federated learning can yield arbitrarily bad results~\citep[cf.][]{shamir2014distributed}.
As the probability of bad local models starkly increases with smaller sample sizes, this trivial bound often carries over to our considered practical settings.
The Radon machine~\citep{kamp2017effective} is a federated learning approach that overcomes these issues for a wide range of learning algorithms and allows us to analyze (non-trivial) quality bounds of aggregated models under the assumption of convexity. Next, we show that \ourmethod can improve model quality for small local datasets where standard federated learning fails to do so.
%
\begin{figure*}[t]
 \centering
 \begin{subfigure}[t]{0.34\textwidth}
 \includegraphics{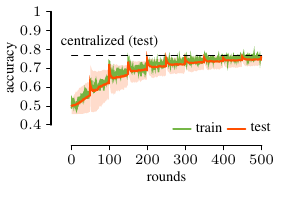}
  \caption{\ourmethod with Radon point with $d=1$, $b=50$.}
\label{fig:radonFedDC}
\end{subfigure}
\hfill
\begin{subfigure}[t]{0.32\textwidth}
\includegraphics{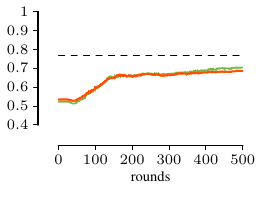}
  \caption{Federated learning with Radon point with $b=1$.}
\label{fig:radonOnly1}
\end{subfigure}
\hfill
\begin{subfigure}[t]{0.32\textwidth}
\includegraphics{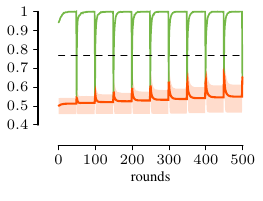}
  \caption{Federated learning with Radon point with $b=50$.}
\label{fig:radonOnly}
\end{subfigure}
\caption{\textit{Results on SUSY.} We visualize results in terms of train (green) and test error (orange) for (a) \ourmethod$(d=1,b=50)$ and standard federated learning using Radon points for aggregation with (b) $b=1$, i.e., the same amount of communication as \ourmethod, and (c) $b=50$, i.e., the same aggregation period as \ourmethod. The network has 441 clients with 2 data points per client. The performance of a central model trained on all data is indicated by the dashed line.
}
\label{fig:fedDCvsRadonOnly}
\end{figure*}
\tikzexternalenable

A Radon point~\citep{radon1921mengen} of a set of points $S$ from a space $\mathcal{X}$ is---similar to the geometric median---a point in the convex hull of $S$ with a high centrality (i.e., a Tukey depth~\citep{tukey1975mathematics, gilad2004bayes} of at least $2$). For a Radon point to exist, $S\subset\mathcal{X}$ has to have a minimum size $r\in\NN$ called the Radon number of $\mathcal{X}$. For $\mathcal{X}\subseteq \RR^d$ the radon number is $d+2$.
%
%
Here, the set of points $S$ are the local models, or more precisely their parameter vectors.
We make the following standard assumption~\citep{von2011statistical} on the local learning algorithm $\algo$.
\begin{assumption}[$(\epsilon,\delta)$-guarantees]
\label{assump:epsdelta}
The learning algorithm $\algo$ applied on a dataset drawn iid from $\Dcal$ of size $\samplesize\geq\samplesize_0\in\NN$ produces a model $\model\in\modelspace$ s.t. with probability $\delta\in (0,1]$ it holds for $\epsilon>0$ that 
$\prob{\risk(\model)>\epsilon}<\delta$.
The sample size $\samplesize_0$ is monotonically decreasing in $\delta$ and $\epsilon$ (note that typically $\samplesize_0$ is a polynomial in $\epsilon^{-1}$ and $\log(\delta^{-1})$).
\end{assumption}
Here $\risk(\model)$ is the risk defined in Sec.~\ref{sec:prelim}. 
Now let $r\in\NN$ be the Radon number of $\modelspace$, $\algo$ be a learning algorithm as in assumption~\ref{assump:epsdelta}, and risk $\risk$ be convex. Assume $m\geq r^h$ many clients with $h\in\NN$. For $\epsilon>0, \delta\in (0,1]$ assume local datasets $D_1,\dots,D_m$ of size larger than $\samplesize_0(\epsilon,\delta)$ drawn iid from $\Dcal$, and $\model_1,\dots,\model_m$ be local models trained on them using $\algo$. Let $\radonpoint_h$ be the iterated Radon point~\citep{clarkson1996approximating} with $h$ iterations computed on the local models (for details, see App.~\ref{app:radonmachine}). Then it follows from Theorem 3 in~\citet{kamp2017effective} that for all $i\in [m]$ it holds that
\begin{equation}
    \prob{\risk(\radonpoint_h)>\epsilon}\leq\left(r\prob{\risk(\model_i)>\epsilon}\right)^{2^h}
    \label{eq:radonPointImprovement}
\end{equation}
where the probability is over the random draws of local datasets. 
That is, the probability that the aggregate $\radonpoint_h$ is bad is doubly-exponentially smaller than the probability that a local model is bad. Note that in PAC-learning, the error bound and the probability of the bound to hold are typically linked, so that improving one can be translated to improving the other~\citep{von2011statistical}. 
Eq.~\ref{eq:radonPointImprovement} implies that the iterated Radon point only improves the guarantee on the confidence compared to that for local models if $\delta < r^{-1}$, i.e. $\prob{\risk(\radonpoint_h)>\epsilon}\leq\left(r\prob{\risk(\model_i)>\epsilon}\right)^{2^h} < \left(r\delta\right)^{2^h}<1$ only holds for $r\delta < 1$. Consequently, local models need to achieve a minimum quality for the federated learning system to improve model quality.
\begin{corollary}
Let $\modelspace$ be a model space with Radon number $r\in\NN$,  $\risk$ a convex risk, and $\algo$ a learning algorithm with sample size $\samplesize_0(\epsilon,\delta)$. Given $\epsilon>0$ and any $h\in\NN$, if local datasets $D_1,\dots,D_m$ with $m\geq r^h$ are smaller than $\samplesize_0(\epsilon,r^{-1})$, then federated learning using the Radon point does not improve model quality in terms of $(\epsilon,\delta)$-guarantees.
\label{lm:fedlearnArbitrarilyBad}
\end{corollary}
In other words, when using aggregation by Radon points alone, an improvement in terms of $(\epsilon,\delta)$-guarantees is strongly dependent on large enough local datasets.
Furthermore, given $\delta>r^{-1}$, the guarantee can become arbitrarily bad by increasing the number of aggregation rounds. 

Federated Daisy-Chaining as given in Alg.~\ref{alg:feddc} permutes local models at random, which is in theory equivalent to permuting local datasets.
Since the permutation is drawn at random, the amount of permutation rounds $T$ necessary for each model to observe a  minimum number of distinct datasets $k$ with probability $1-\rho$ can be given with high probability via a variation of the coupon collector problem as 
$
T\geq d\frac{m}{\rho^\frac1m}(H_m-H_{m-k})
$, 
where $H_m$ is the $m$-th harmonic number---see Lm.~\ref{lm:minNumberRoundsDC} in App.~\ref{app:prooffeddcworks} for details.
It follows that when we perform daisy-chaining with $m$ clients and local datasets of size $n$ for at least $ dm\rho^{-\frac{1}{m}}(H_m-H_{m-k})$ rounds, then each local model will with probability at least $1-\rho$ be trained on at least $kn$ distinct samples. For an $\epsilon,\delta$-guarantee, we thus need to set $b$ large enough so that $kn\geq n_0(\epsilon,\sqrt{\delta})$ with probability at least $1-\sqrt{\delta}$. This way, the failure probability is the product of not all clients observing $k$ distinct datasets and the model having a risk larger than $\epsilon$, which is $\sqrt{\delta}\sqrt{\delta}=\delta$.
\begin{proposition}
Let $\modelspace$ be a model space with Radon number $r\in\NN$, $\risk$ a convex risk , and $\algo$ a learning algorithm with sample size $\samplesize_0(\epsilon,\delta)$. Given $\epsilon>0$, $\delta\in (0,r^{-1})$ and any $h\in\NN$, and local datasets $D_1,\dots,D_m$ of size $n\in\NN$ with $m\geq r^h$, then Alg.~\ref{alg:feddc} using the Radon point with aggr. period
\begin{equation}
b\geq  d\frac{m}{\delta^\frac{1}{2m}}\left(H_m-H_{m-\left\lceil n^{-1}\samplesize_0\left(\epsilon,\sqrt{\delta}\right)\right\rceil}\right)
\label{eq:aggToDaisychainRounds}
\end{equation}
improves model quality in terms of $(\epsilon,\delta)$-guarantees.
\label{prop:fedDCworks}
\end{proposition}
%
This result implies that if enough daisy-chaining rounds are performed in-between aggregation rounds, 
federated learning via the iterated Radon point improves model quality in terms of $(\epsilon,\delta)$-guarantees: the resulting model has generalization error smaller than $\epsilon$ with probability at least $1-\delta$. Note that the aggregation period cannot be arbitrarily increased without harming convergence. To illustrate the interplay between these variables, we provide a numerical analysis of Prop.~\ref{prop:fedDCworks} in App.~\ref{app:numericalanalysisProp5}.

This theoretical result is also evident in practice, as we show in Fig.~\ref{fig:fedDCvsRadonOnly}. There, we compare \ourmethod with standard federated learning and equip both with the iterated Radon point  on the SUSY binary classification dataset~\citep{baldi2014searching}. We train a linear model on $441$ clients with only $2$ samples per client. After $500$ rounds \ourmethod daisy-chaining every round ($d=1$) and aggregating every fifty rounds ($b=50$) reached the test accuracy of a  gold-standard model that has been trained on the centralized dataset (ACC=$0.77$).  Standard federated learning with the same communication complexity using $b=1$ is outperformed by a large margin (ACC=$0.68$). We additionally provide results of standard federated learning with $b=50$ (ACC=$0.64$), which shows that while the aggregated models perform reasonable, the standard approach heavily overfits on local datasets if not pulled to a global average in every round. More details on this experiment can be found in App.~\ref{app:radonsusy}. In Sec.~\ref{sec:experiments} we show that the empirical results for averaging as aggregation operator are similar to those for the Radon machine. First, we discuss the privacy-aspects of \ourmethod.

	\section{Data Privacy}
\label{sec:privacy}
\begin{wrapfigure}{r}{6.1cm}
\centering
\includegraphics{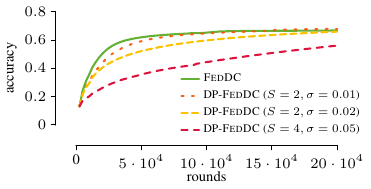}
\caption{\textit{Differential privacy results.} Comparison of \ourmethod (top solid line) to \ourmethod with clipped parameter updates and Gaussian noise (dashed lines) on CIFAR10 with $250$ clients.}
\label{fig:diffprivacy}
\end{wrapfigure}
\tikzexternalenable
A major advantage of federated over centralized learning is that local data remains undisclosed to anyone but the local client,  only model parameters are exchanged. This provides a natural benefit to data privacy, which is the main concern in applications such as healthcare. However, an attacker can make inferences about local data from model parameters~\citep{ma2020safeguarding} and model updates or gradients~\citep{zhu2020deep}. In the daisy-chaining rounds of \ourmethod clients receive a model that was directly trained on the local data of another client, instead of a model aggregate, potentially facilitating membership inference attacks~\citep{shokri2017membership}---reconstruction attacks~\citep{zhu2020deep} remain difficult because model updates cannot be inferred since the server randomly permutes the order of clients in daisy-chaining rounds.
Should a malicious client obtain model updates through additional attacks, a common defense is applying appropriate clipping and noise before sending models. This guarantees $\epsilon,\delta$-differential privacy for local data~\citep{wei2020federated} at the cost of a slight-to-moderate loss in model quality. This technique is also proven to defend against backdoor and poisoning attacks~\citep{sun2019can}. Moreover, \feddc is compatible with standard defenses against such attacks, such as noisy or robust aggregation~\citep{liu_threats_2022}---\ourmethod with the Radon machine is an example of robust aggregation.
We illustrate the effectiveness of \ourmethod with differential privacy in the following experiment. We train a small ResNet on $250$ clients using \ourmethod with $d=2$ and $b=10$, postponing the details on the experimental setup to App.~\ref{app:architectures} and~\ref{app:training}. Differential privacy is achieved by clipping local model updates and adding Gaussian noise as proposed by~\citet{geyer2017differentially}. The results as shown in Figure~\ref{fig:diffprivacy} indicate that the standard trade-off between model quality and privacy holds for \ourmethod as well. Moreover, for mild privacy settings the model quality does not decrease. That is, \ourmethod is able to robustly predict even under differential privacy. We provide an extended discussion on the privacy aspects of \ourmethod in App.~\ref{app:attacks}.

	\section{Experiments on Deep Learning}
\label{sec:experiments}

Our approach \ourmethod, both provably and empirically, improves model quality when using Radon points as aggregation which, however, require convex problems.
For non-convex problems, in particular deep learning, averaging is the state-of-the-art aggregation operator.
We, hence, evaluate \ourmethod with averaging against the state of the art in federated learning on synthetic and real world data using neural networks.
%
As baselines, we consider federated averaging (\fedavg)~\citep{mcmahan2017communication} with optimal communication, \fedavg with  equal communication as \ourmethod, and simple daisy-chaining without aggregation.
We further consider the $4$ state-of-the-art methods \fedprox~\citep{li2020federated}, \fedadagrad, \fedyogi, and \fedadam~\citep{reddi2020adaptive}.
As datasets we consider a synthetic classification dataset, image classification in CIFAR10~\citep{krizhevsky2009learning}, and two real medical datasets: MRI scans for brain tumors,\!\footnote{kaggle.com/navoneel/brain-mri-images-for-brain-tumor-detection} and chest X-rays for pneumonia\footnote{kaggle.com/praveengovi/coronahack-chest-xraydataset}. We provide additional results on MNIST in App.~\ref{app:mnist}.
Details on the experimental setup are in App.~\ref{app:architectures},\ref{app:training}, code is publicly available at~\url{\oururl}.

\paragraph{Synthetic Data:}
\begin{figure*}[t]
 \centering
 \begin{subfigure}[t]{0.34\textwidth}
 \includegraphics{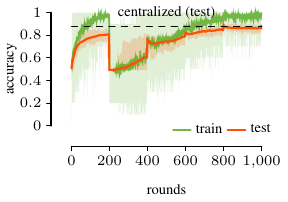}
  \caption{\ourmethod with $d=1$, $b=200$.}
\label{fig:synthMLPFedDCvsFedAvg:feddc}
\end{subfigure}
\hfill
 \begin{subfigure}[t]{0.32\textwidth}
 \includegraphics{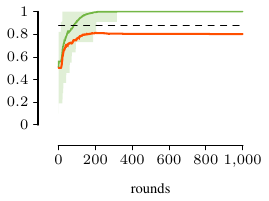}
  \caption{\fedavg with $b=1$.}
\label{fig:synthMLPFedDCvsFedAvg1:avg}
\end{subfigure}
\hfill
\begin{subfigure}[t]{0.32\textwidth}
\includegraphics{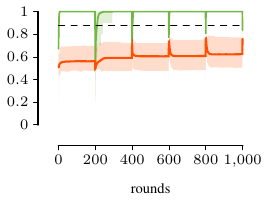}
\caption{\fedavg with $b=200$.}
\label{fig:synthMLPFedDCvsFedAvg:avg}
\end{subfigure}
\caption{\textit{Synthetic data results.} Comparison of \ourmethod (a), \fedavg with same communication (b) and same averaging period (c) for training fully connected NNs on synthetic data. We report mean and confidence accuracy per client in color and accuracy of central learning as dashed black line.}
\label{fig:synthMLPFedDCvsFedAvg}
\end{figure*}
\tikzexternalenable
We first investigate the potential of \ourmethod on a synthetic binary classification dataset generated by the sklearn~\citep{pedregosa2011scikit} \verb+make_classification+ function with $100$ features. On this dataset, we train a simple fully connected neural network with $3$ hidden layers on $m=50$ clients with $n=10$ samples per client.
We compare \ourmethod with daisy-chaining period $d=1$ and aggregation period $b=200$ to \fedavg with the same amount of communication $b=1$ and the same averaging period $b=200$. The results presented in Fig.~\ref{fig:synthMLPFedDCvsFedAvg} show that \ourmethod achieves a test accuracy of $0.89$. This is comparable to centralized training on all data which achieves a test accuracy of $0.88$. It substantially outperforms both \fedavg setups, which result in an accuracy of $0.80$ and $0.76$. Investigating the training of local models between aggreation periods reveals that the main issue of \fedavg is overfitting of local clients, where \fedavg \textit{train} accuracy reaches $1.0$ quickly after each averaging step. 
With these promising results on vanilla neural networks, we next turn to real-world image classification problems typically solved with CNNs.

\paragraph{CIFAR10:}
As a first challenge for image classification, we consider the well-known CIFAR10 image benchmark.
We first investigate the effect of the aggregation period $b$ on \ourmethod and \fedavg, separately optimizing for an optimal period for both methods.
We use a setting of $250$ clients with a small version of ResNet, and $64$ local samples each, which simulates our small sample setting, drawn at random without replacement (details in App.~\ref{app:training}).
We report the results in Figure~\ref{fig:avgPeriod} and set the period for \ourmethod to $b=10$, and consider federated averaging with periods of both $b=1$ (equivalent communication to \ourmethod with $d=1,b=10$) and $b=10$ (less communication than \ourmethod by a factor of $10$) for all subsequent experiments.

\begin{wrapfigure}{r}{7.0cm}
 \centering
 \includegraphics{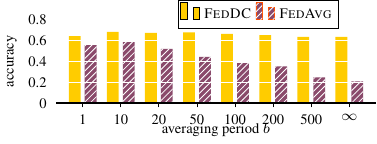}

\caption{\textit{Averaging periods on CIFAR10.} For 150 clients with small ResNets and 64 samples per client, we visualize the test accuracy (higher is better) of \ourmethod and \fedavg for different aggregation periods $b$.
}
\label{fig:avgPeriod}
\end{wrapfigure}
\tikzexternalenable
Next, we consider a subset of $9600$ samples spread across $150$ clients (i.e. $64$ samples per client), which corresponds to our small sample setting. Now,  each client is equipped with a larger, untrained ResNet18.\!\footnote{Due to hardware restrictions we are limited to training 150 ResNets, hence 9600 samples across 150 clients.}
%
Note that the combined amount of examples is only one fifth of the original training data, hence we cannot expect typical CIFAR10 performance.
To obtain a gold standard for comparison, we run centralized learning \central, separately optimizing its hyperparameters, yielding an accuracy of around $0.65$.
All results are reported in Table~\ref{tab:main_results}, where we report \fedavg with $b=1$ and $b=10$, as these were the best performing settings and $b=1$ corresponds to equal amounts of communication as \feddc.
We use a daisy chaining period of $d=1$ for \feddc throughout all experiments for consistency, and provide results for larger daisy chaining periods in App.~\ref{app:cifarothers}, which, depending on the data distribution, might be favorable.
We observe that \ourmethod achieves substantially higher accuracy over the baseline set by federated averaging. In App.~\ref{app:clientsubsampling} we show that this holds also for client subsampling.
Upon further inspection, we see that \fedavg drastically overfits, achieving training accuracies of $0.97$ (App.~\ref{app:cifartrainacc}), a similar trend as on the synthetic data before.
Daisy-chaining alone, apart from privacy issues,  also performs worse than \ourmethod. 
Intriguingly, also the state of the art shows similar trends. 
\fedprox, run with optimal $b=10$ and $\mu=0.1$, only achieves an accuracy of $0.51$ and \fedadagrad, \fedyogi, and \fedadam show even worse performance of around $0.22$, $0.31$, and $0.34$, respectively. While applied successfully on large-scale data, these methods seem to have shortcomings when it comes to small sample regimes.
%

To model different data distributions across clients that could occur in for example our healthcare setting, we ran further experiments on simulated non-iid data, gradually increasing the locally available data, as well as on non-privacy preserving decentralized learning.
We investigate the effect of non-iid data on \ourmethod by studying the ``pathological non-IID partition of the data''~\citep{mcmahan2017communication}. Here, each client only sees examples from $2$ out of the $10$ classes of CIFAR10. We again use a subset of the dataset. The results in Tab.~\ref{tab:noniid} show that \ourmethod outperforms \fedavg by a wide margin. It also outperforms \fedprox, a method specialized on heterogeneous datasets in our considered small sample setting. 
For a similar training setup as before, we show results for  gradually increasing local datasets in App.~\ref{app:localdatasetsize}. Most notably, \feddc outperforms \fedavg even with 150 samples locally. Only when the full CIFAR10 dataset is distributed across the clients, \fedavg is on par with \feddc (see App. Fig.~\ref{fig:dataset_size}).
We also compare with distributed training through gradient sharing (App.~\ref{app:minibatchSGD}), which discards any privacy concerns, implemented by mini-batch SGD with parameter settings corresponding to our federated setup as well as a separately optimized version. The results show that such an approach is outperformed by both \fedavg as well as \feddc, which is in line with previous findings and emphasize the importance of model aggregation.

As a final experiment on CIFAR10, we consider daisy-chaining with different combinations of aggregation methods, and hence its ability to serve as a building block that can be combined with other federated learning approaches.
In particular, we consider the same setting as before and combine \fedprox with daisy chaining.
The results, reported in Tab.~\ref{tab:noniid}, show that this combination is not only successful, but also outperforms all others in terms of accuracy.

\begin{table}[t]
\parbox[t]{.59\linewidth}{\hrule height 0pt width 0pt 
    \footnotesize
    \setlength\tabcolsep{5pt}
    \begin{tabular}{lccc}
        \toprule
         & \textbf{CIFAR10} & \textbf{MRI} & \textbf{Pneumonia}\\
        \midrule
        \ourmethod (ours) & $\mathbf{62.9}$ \hspace{-0.5pt}\scriptsize{$\mathbf{\pm 0.02}$} & $\mathbf{78.4}$ \hspace{-0.5pt}\scriptsize{$\mathbf{\pm 0.61}$} & $\mathbf{83.2}$ \hspace{-0.5pt}\scriptsize{$\mathbf{\pm 0.84}$}\\[0.05cm]
        DC (baseline) & $58.4$ \scriptsize{$\pm 0.85$} & $57.7$ \scriptsize{$\pm 1.57$} & $79.8$ \scriptsize{$\pm 0.99$}\\[0.05cm]
        \fedavg(b=1) & $55.8$ \scriptsize{$\pm 0.78$} & $74.1$ \scriptsize{$\pm 1.68$} & $80.1$ \scriptsize{$\pm 1.53$}\\[0.05cm]
        \fedavg(b=10) & $48.7$ \scriptsize{$\pm 0.87$} & $75.6$ \scriptsize{$\pm 1.18$} & $79.4$ \scriptsize{$\pm 1.11$}\\[0.05cm]
        {\fedprox} & $51.1$ \scriptsize{$\pm 0.80$} & $76.5$ \scriptsize{$\pm 0.50$} & $80.0$ \scriptsize{$\pm 0.36$} \\[0.05cm]
        {\fedadagrad} & $21.8$ \scriptsize{$\pm 0.01$} & $45.7$ \scriptsize{$\pm 1.25$} & $62.5$ \scriptsize{$\pm 0.01$} \\[0.05cm]
        {\fedyogi} & $31.4$ \scriptsize{$\pm 4.37$} & $71.3$ \scriptsize{$\pm 1.62$} & $77.6$ \scriptsize{$\pm 0.64$} \\[0.05cm]
        {\fedadam} & $34.0$ \scriptsize{$\pm 0.23$} & $73.8$ \scriptsize{$\pm 1.98$} & $73.5$ \scriptsize{$\pm 0.36$} \\[0.05cm]
        \midrule
        {\central} & $65.1$ \scriptsize{$\pm 1.44$} & $82.1$ \scriptsize{$\pm 1.00$} & $84.1$ \scriptsize{$\pm 3.31$} \\[0.05cm]
        \bottomrule
    \end{tabular}
    \vspace{0.2cm}
    \caption{Results on image data, reported is the average test accuracy of the final model over three runs ($\pm$ denotes maximum deviation from the average).}
    \label{tab:main_results}\vfill
}\hfill
\parbox[t]{.36\linewidth}{\hrule height 0pt width 0pt 
    
    
    \begingroup
    \footnotesize
    \setlength\tabcolsep{4pt}
    \begin{tabular}{lc}
        \toprule
        & \textbf{CIFAR10}\\
        \midrule
        \ourmethod & $\mathbf{62.9}$ \scriptsize{$\mathbf{\pm 0.02}$} \\[0.05cm]
        \ourmethod+\fedprox & $\mathbf{63.2}$ \scriptsize{$\mathbf{\pm 0.38}$}\\[0.05cm]
        \toprule
         \textbf{} & \textbf{Non-IID}\\
        \midrule
        \ourmethod & $\mathbf{34.2}$ \scriptsize{$\mathbf{\pm 0.61}$}\\[0.05cm]
        \fedavg(b=1) & $30.2$ \scriptsize{$\pm 2.11$}\\[0.05cm]
        \fedavg(b=10) & $24.9$ \scriptsize{$\pm 1.95$}\\[0.05cm]
        \fedprox & $32.8$ \scriptsize{$\pm 0.00$}\\[0.05cm]
        \fedadagrad & $11.7$ \scriptsize{$\pm 0.00$}\\[0.05cm]
        \fedadam & $13.0$ \scriptsize{$\pm 0.00$}\\[0.05cm]
        \fedyogi & $12.5$ \scriptsize{$\pm 0.04$}\\[0.05cm]
        \bottomrule
    \end{tabular}
    \endgroup
    \vspace{0.1cm}
    \caption{Combination of \ourmethod with \fedavg and \fedprox and non-iid results on CIFAR10.}
    \label{tab:noniid}
}
\end{table}

\paragraph{Medical image data:}
Finally, we consider two real medical image datasets representing actual health related machine learning tasks, which are naturally of small sample size. 
For the brain MRI scans, we simulate $25$ clients (e.g.,  hospitals) with $8$ samples each. Each client is equipped with a CNN (see App. \ref{app:architectures}). The results for brain tumor prediction evaluated on a test set of $53$ of these scans are reported in Table~\ref{tab:main_results}. 
Overall, \ourmethod performs best among the federated learning approaches and is close to the centralized model.
Whereas \fedprox performed comparably poorly on CIFAR10, it now outperforms \fedavg.
Similar to before, we observe a considerable margin between all competing methods and \ourmethod.
To investigate the effect of skewed distributions of sample sizes across clients, such as smaller hospitals having less data than larger ones, we provide additional experiments in App.~\ref{app:cifarothers}. The key insight is that also in these settings, \feddc outperforms \fedavg considerably, and is close to its performance on the unskewed datasets.

For the pneumonia dataset, we simulate $150$ clients training ResNet18 (see App. \ref{app:architectures}) with $8$ samples per client, the hold out test set are $624$ images. The results, reported in Table~\ref{tab:main_results}, show  similar trends as for the other datasets, with \ourmethod outperforming all baselines and the state of the art, and being within the performance of the centrally trained model. Moreover it highlights that \ourmethod enables us to train a ResNet18 to high accuracy with as little as $8$ samples per client.

	\section{Discussion and Conclusion}
\label{sec:discussion}
We propose to combine daisy-chaining and aggregation to effectively learn high quality models in a federated setting where only little data is available locally.
We formally prove convergence of our approach \ourmethod, and for convex settings provide PAC-like generalization guarantees when aggregating by iterated Radon points. Empirical results on the SUSY benchmark underline these theoretical guarantees, with \ourmethod matching the performance of centralized learning.
Extensive empirical evaluation shows that the proposed combination of daisy-chaining and aggregation enables federated learning from small datasets in practice.
When using averaging, we improve upon the state of the art for federated deep learning by a large margin for the considered small sample settings. 
Last but not least, we show that daisy-chaining is not restricted to \feddc, but can be straight-forwardly included in \fedavg, Radon machines, and \fedprox as a building block, too.

\ourmethod permits differential privacy mechanisms that introduce noise on model parameters, offering protection against membership inference, poisoning and backdoor attacks. Through the random permutations in daisy-chaining rounds, \ourmethod is also robust against reconstruction attacks.
Through the daisy-chaining rounds, we see a linear increase in communication. As we are primarily interested in healthcare applications, where communication is not a bottleneck, such an increase in communication is negligible. 
Importantly, \ourmethod outperforms \fedavg in practice also when both use the same amount of communication.
Improving the communication efficiency considering settings where bandwidth is limited, e.g., model training on mobile devices, would make for engaging future work.

We conclude that daisy-chaining lends itself as a simple, yet effective building block to improve federated learning, complementing existing work to extend to settings where little data is available per client.
\ourmethod, thus, might offer a solution to the open problem of federated learning in healthcare, where very few, undisclosable samples are available at each site.

    \vfill
    \pagebreak
    \subsubsection*{Acknowledgments}
    The authors thank Sebastian U. Stich for his detailed comments on an earlier draft. 
    Michael Kamp received support from the Cancer Research Center Cologne Essen (CCCE).
    Jonas Fischer is supported by a grant from the US National Cancer Institute (R35CA220523).
    {
        \bibliographystyle{plainnat}
        \bibliography{references}
    }
    \vfill
    \pagebreak
    \appendix
	\section{Appendix}
\label{sec:apx}

\subsection{Details on Experimental Setup}
\label{app:expdetails}
In this section we provide all details to reproduce the empirical results presented in this paper. Furthermore, the implementation provided at~\url{\oururl} allows to directly reproduce the result.
Experiments were conducted on an NVIDIA DGX with six A6000 GPUs. 

\subsubsection{Network architectures}
\label{app:architectures}
Here, we detail network architectures considered in our empirical evaluation.

\paragraph{MLP for Synthetic Data}
A standard multilayer perceptron (MLP) with ReLU activations and three linear layers of size $100$,$50$,$20$.

\paragraph{Averaging round experiment}
 For this set of experiments we use smaller versions of ResNet architectures with 3 blocks, where the blocks use $16,32,64$ filters, respectively. In essence, these are smaller versions of the original ResNet18 to keep training of 250 networks feasible.

\paragraph{CIFAR10 \& Pneumonia} For CIFAR10, we consider a standard ResNet18 architecture, where weights are initialized by a Kaiming Normal and biases are zero-initialized. Each client constructs and initializes a ResNet network separately. For pneumonia, X-ray images are resized to $(224,224)$.

\paragraph{MRI} For the MRI data, we train a small convolutional network of architecture Conv(32)-Batchnorm-ReLU-MaxPool-Conv(64)-Batchnorm-ReLU-MaxPool-Linear, where Conv($x$) are convolutional layers with $x$ filters of kernel size 3.
The pooling layer uses a stride of $2$ and kernel size of $2$. The Linear layer is of size $2$ matching the number of output classes. All scan images are resized to $(150,150)$.

\subsubsection{Training setup}
\label{app:training}

In this section, we give additional information for the training setup for each individual experiment of our empirical evaluation.

\paragraph{SUSY experiments}
SUSY is a binary classification dataset with $18$ features. We train linear models with stochastic gradient descent (learning rate $0.0001$, found by grid-search on an independent part of the dataset) on $441$ clients, aggregating every $50$ rounds via the iterated Radon point~\citep{kamp2017effective} with $h=2$ iterations. \ourmethod performs daisy-chaining with period $d=1$. The test accuracy is evaluated on a test set with $1\,000\,000$ samples drawn iid at random.

\paragraph{Synthetic Data}
The synthetic binary classification dataset is generated by the sklearn~\citep{pedregosa2011scikit} \verb+make_classification+ function with $100$ features of which $20$ are informative, $60$ are redundant, and $5$ are repeated. We generate $3$ clusters per class with a class separation of $1.0$, a shift of $1.0$ and a scale of $3.0$. Class labels are randomly flipped with probability $0.02$.

\paragraph{Averaging rounds parameter optimization} To find a suitable number when averaging should be carried out, we explore $b\in\{1,10,20,50,100,200,500,\infty\}$ on CIFAR10 using $250$ clients each equipped with a small ResNet.
We assign 64 samples to each client drawn at random (without replacement) from the CIFAR10 training data and use a batch size of $64$. For each parameter, we train for $10k$ rounds with SGD using cross entropy loss and initial learning rate of $0.1$, multiplying the rate by a factor of $.5$ every $2500$ rounds.

\paragraph{FedAdam, FedAdagrad, and FedYogi}
We use the standard values for $\beta_1$ and $\beta_2$, i.e., $\beta_1=0.9$, $\beta_2=0.999$, as suggested in~\citet{reddi2020adaptive}. We optimized learning rate $\eta_l$ and global learning rate $\eta$ from the set $\{0.001,0.01,0.1,1.0,2.0\}$ yielding optimal parameters $\eta_l=0.1$ and $\eta=1.0$. 

\paragraph{CIFAR10 differential privacy and main experiments}

We keep the same experimental setup as for hyperparameter tuning, but now use $100$ clients each equipped with a ResNet18. 

\subsection{Iterated Radon Points and the Radon Machine}
\label{app:radonmachine}
The Radon machine~\citep{kamp2017effective} aggregates a set $S={h^1,\dots,h^m}$ of local models $h^i\in\modelspace$ via the iterated Radon point algorithm~\citep{clarkson1996approximating}. For models with $d\in\NN$ parameters, $r=d+2$ many models are required to compute a single Radon point, where $r$ is called the Radon number of $\modelspace$. Let $m=r^h$ for some $h\in\NN$, then the iterated Radon point aggregates models in $h$ iterations. In each iteration, the set $S$ is partitioned into subsets of size $r$ and the Radon point of each subset is calculated. The final step of each iteration is to replace the set $S$ of models by the set of Radon points. After $h$ iterations, a single Radon point $\mathfrak{r}_h$ is obtained as the aggregate. 

Radon points can be obtained by solving a system of linear equations of size $r\times r$~\citep{kamp2017effective}: In his main theorem, \citet{radon1921mengen} gives the following construction of a Radon point for a set $S=\{s_1,...,s_r\}\subseteq\RR^d$. Find a non-zero solution $\lambda\in\RR^{|S|}$ for the following linear equations.
\[
\sum_{i=1}^r\lambda_i s_i = (0,\dots,0)\enspace , \enspace \sum_{i=1}^r \lambda_i = 0
\]
Such a solution exists, since $|S| > d+1$ implies that $S$ is linearly dependent. Then, let $I,J$ be index sets such that for all $i\in I: \lambda_i\geq 0$ and for all $j\in J: \lambda_j < 0$. Then a Radon point is defined by
\[
\mathfrak{r}(\lambda) = \sum_{i\in I}\frac{\lambda_i}{\Lambda}s_i = \sum_{j\in J}\frac{\lambda_j}{\Lambda}s_j\enspace ,
\]
where $\Lambda = \sum_{i\in I}\lambda_i = -\sum_{j\in J}\lambda_j$. Any solution to this linear system of equations is a Radon point. The equation system can be solved in time $r^3$. By setting the first element of $\lambda$ to one, we obtain a unique solution of the system of linear equations. Using this solution $\lambda$, we define the Radon point of a set $S$ as $\mathfrak{r}(S)=\mathfrak{r}(\lambda)$ in order to resolve ambiguity.






\subsection{Additional Empirical Results}
In Sec.~\ref{sec:experiments} we have shown that \ourmethod performs well on benchmark and real-world datasets. In the following we provide additional empirical results, both to investigate the main results more closely, as well as to further investigate the properties of \ourmethod. For the CIFAR10 experiment, we investigate training accuracy (App.~\ref{app:cifartrainacc}) and present results for distributed mini-batch SGD (App.~\ref{app:minibatchSGD}). For the SUSY experiment, we compare to \fedavg (App.~\ref{app:radonsusy}). As additional experiments, we investigate the impact of local dataset size (App.~\ref{app:localdatasetsize}) and skewed dataset size distributions (App.~\ref{app:cifarothers}), and analyze the communication-efficiency of \feddc (App.~\ref{app:commeff}). Finally, we present results on MNIST where \ourmethod achieves state-of-the-art accuracy (App.~\ref{app:mnist}).

\subsubsection{Train and test accuracies on CIFAR10:}
\label{app:cifartrainacc}
In Table~\ref{tab:traintest} we provide the accuracies on the entire training set for the final model, together with test accuracies, on CIFAR10. The high training accuracies of \fedavg ($\approx 0.97$)---and to a lesser degree \fedprox ($0.96$)---indicate overfitting on local data sets. The poor training performance of \fedadagrad, \fedyogi, and \fedadam hint at insufficient model updates. A possible explanation is that the adaptive learning rate parameter (which is proportional to the sum of past model updates) becomes large quickly, essentially stopping the training process. The likely reason is that due to large differences in local data distributions, model updates after each aggregation round are large.
\begin{table}[ht]
    \centering
    \begingroup
    \footnotesize
    \setlength\tabcolsep{5pt}
    \begin{tabular}{lcc}
        \toprule
         & \textbf{Test} & \textbf{Train} \\
        \midrule
        \ourmethod (ours) & ${62.9}$ \scriptsize{${\pm 0.02}$} & ${94.7}$ \scriptsize{${\pm 0.52}$}\\[0.05cm]
        DC (baseline) & $58.4$ \scriptsize{$\pm 0.85$} & $94.1$ \scriptsize{$\pm 2.31$} \\[0.05cm]
        \fedavg(b=1) & $55.8$ \scriptsize{$\pm 0.78$} & $97.2$ \scriptsize{$\pm 0.87$} \\[0.05cm]
        \fedavg(b=10) & $48.7$ \scriptsize{$\pm 0.87$} & $97.4$ \scriptsize{$\pm 0.23$} \\[0.05cm]
        {\fedprox} & $51.1$ \scriptsize{$\pm 0.80$} & $95.9$ \scriptsize{$\pm 0.42$} \\[0.05cm]
        {\fedadagrad} & $21.8$ \scriptsize{$\pm 0.01$} & $31.7$ \scriptsize{$\pm 0.25$} \\[0.05cm]
        {\fedyogi} & $31.4$ \scriptsize{$\pm 4.37$} & $72.4$ \scriptsize{$\pm 0.90$} \\[0.05cm]
        {\fedadam} & $34.0$ \scriptsize{$\pm 0.23$} & $73.9$ \scriptsize{$\pm 0.89$} \\[0.05cm]
        \bottomrule
    \end{tabular}
    \endgroup
    \vspace{0.3cm}
    \caption{Train and test accuracy on CIFAR10 of the final model over three runs ($\pm$ denotes maximum deviation from the average).}
    \label{tab:traintest}
\end{table}



\subsubsection{Additional Results on SUSY}
\label{app:radonsusy}
In Sec.~\ref{sec:theory} we compared \ourmethod to federated learning with the iterated Radon point. For completeness, we compare it to \fedavg as well, i.e.,~federated learning using averaging on the same SUSY binary classification dataset~\citep{baldi2014searching}. The results shown in Fig.~\ref{fig:fedDCvsRadonOnly1} are in line with the findings in Sec.~\ref{sec:theory}: \ourmethod with the iterated Radon point outperforms \fedavg both with the same amount of communication ($b=1$) and the same aggregation period ($b=50$). The results for $b=50$ show that \fedavg exhibits the same behavior on local training sets that indicates overfitting. Overall, \fedavg performs comparably to federated learning with the iterated Radon point. That is, \fedavg ($b=1$ and $b=50$) has an accuracy of $0.67 $, resp. $0.64$, compared to $0.68$, resp. $0.64$ for federated learning with the iterated Radon point.
\begin{figure*}[t]
 \centering
\begin{subfigure}[t]{0.34\textwidth}
 \includegraphics{figs/radonfeddc.pdf}
  \caption{\ourmethod with Radon point with $d=1$, $b=50$.}
\label{fig:app:radonFedDC}
\end{subfigure}
\hfill
\begin{subfigure}[t]{0.32\textwidth}
\includegraphics{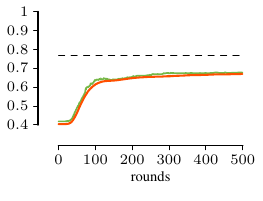}
  \caption{Federated learning with averaging with $b=1$.}
\label{fig:app:avgOnly1}
\end{subfigure}
\begin{subfigure}[t]{0.32\textwidth}
\includegraphics{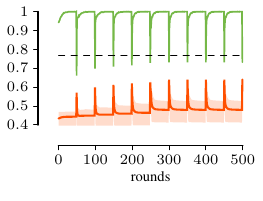}
  \caption{Federated learning with averaging with $b=50$.}
\label{fig:avgOnly}
\end{subfigure}
\caption{\textit{Results on SUSY.} We visualize results in terms of train (green) and test error (orange) for \ourmethod with the iterated Radon point (a) and \fedavg with $b=1$ (b) as well as  \fedavg with $b=50$ (c). The network has 441 clients with 2 data points per client. The performance of a central model trained on all data is indicated by the dashed line.
}
\label{fig:fedDCvsRadonOnly1}
\end{figure*}
\tikzexternalenable

\subsubsection{Comparison with Distributed Mini-Batch SGD on CIFAR10}
\label{app:minibatchSGD}


We compare to distributed mini-batch SGD, i.e., central updates where gradients are computed distributedly on CIFAR10. We use the same setup as for the other experiments, i.e., $m=150$ clients and a mini-batch size of $B=64$, so that the effective mini-batch size for each update is $mB=9600$, the optimal learning rate is $\lambda=0.01$. Here, mini-batch SGD achieves a test accuracy of $19.47 \pm 0.68$. Since a plausible explanation for the poor performance is the large mini-batch size, we compare it to the setting with $B=1$ to achieve the minimum effective mini-batch size of $B=150$. The results are substantially improved to an accuracy of $50.14 \pm 0.63$, underlining the negative effect of the large batch size in line with the theoretical analysis of \citet{shamir2014distributed}. Running it $64$ times the number of rounds that FedDC uses improves the accuracy just slightly to $54.3$. Thus, even with optimal $B=1$ and a 64-times slower convergence, mini-batch SGD is outperformed by both FedAvg and FedDC, since it cannot use the more favorable mini-batch size of $B=64$ on $m=150$ clients.

\subsubsection{Local Dataset Size}
\label{app:localdatasetsize}
In our experiments, we used dataset sizes common in the medical domain, e.g., for radiological images. To further investigate the impact of local dataset sizes on the performance of \ourmethod wrt. \fedavg, we evaluate the performance for local dataset sizes ranging from $2$ to $256$ (given the size of CIFAR10, 256 is the maximum without creating overlap between local datasets). The results in Fig.~\ref{fig:dataset_size} show that \ourmethod outperforms all baselines for smaller local datasets. Only for as much as $256$ examples \fedavg performs as good as \feddc.

These results further confirm that \ourmethod is capable of handling heterogeneous data: for $n<10$ the clients only see a subset of labels due to the size of their local datasets (with n=2, each client can at most observe two classes). We find this a more natural non-iid setting. These results indicate that the shuffle mechanism indeed mitigates data heterogeneity well. A further study of the impact of non-iid data, a comparison with personalized FL, and potential improvements to the shuffling scheme are interesting directions for future work
\begin{figure}[h]
\begin{subfigure}[t]{0.45\textwidth}
\centering
    \includegraphics[width=0.9\textwidth]{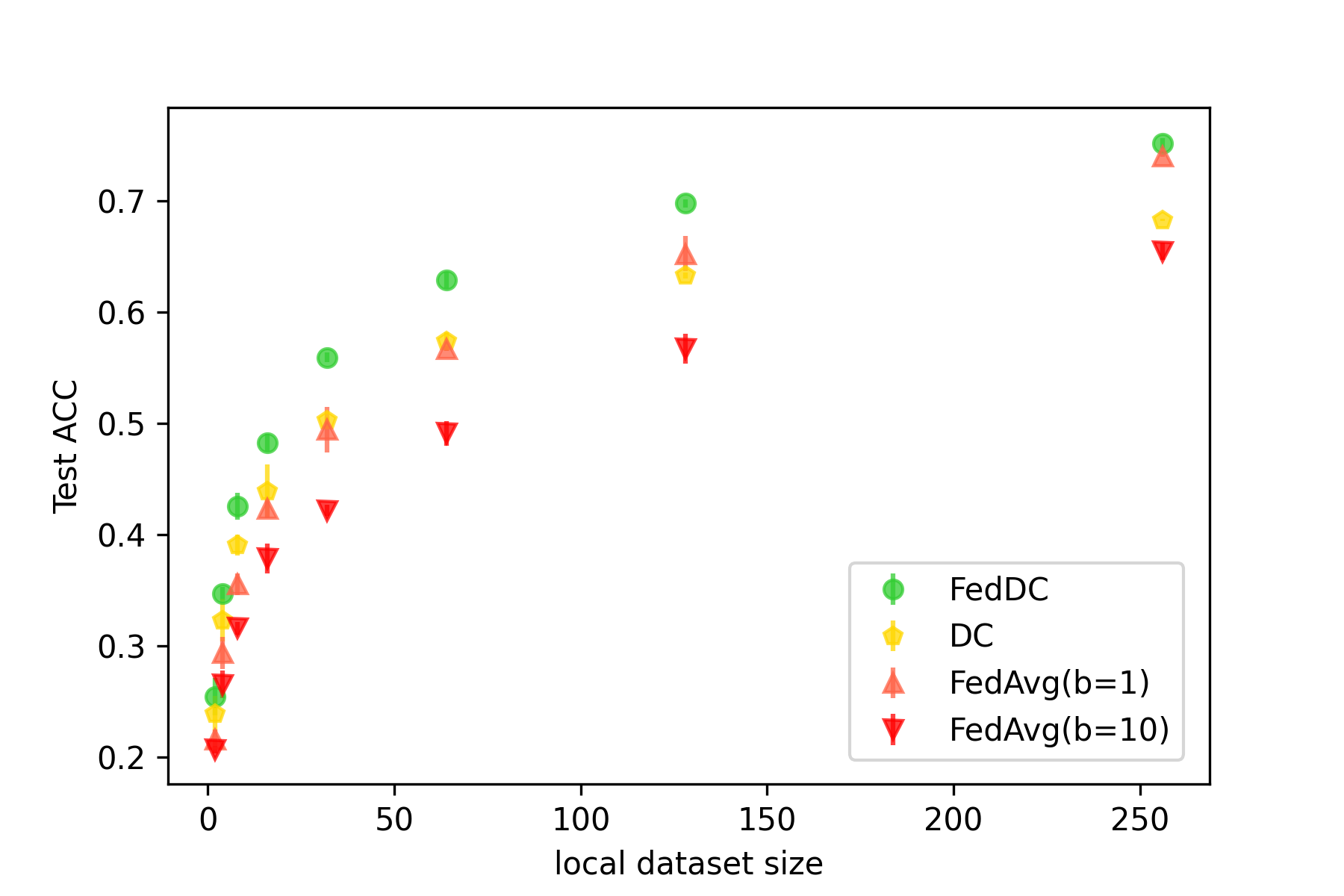}
\end{subfigure}
\begin{subfigure}[t]{0.45\textwidth}
    \includegraphics[width=0.9\textwidth]{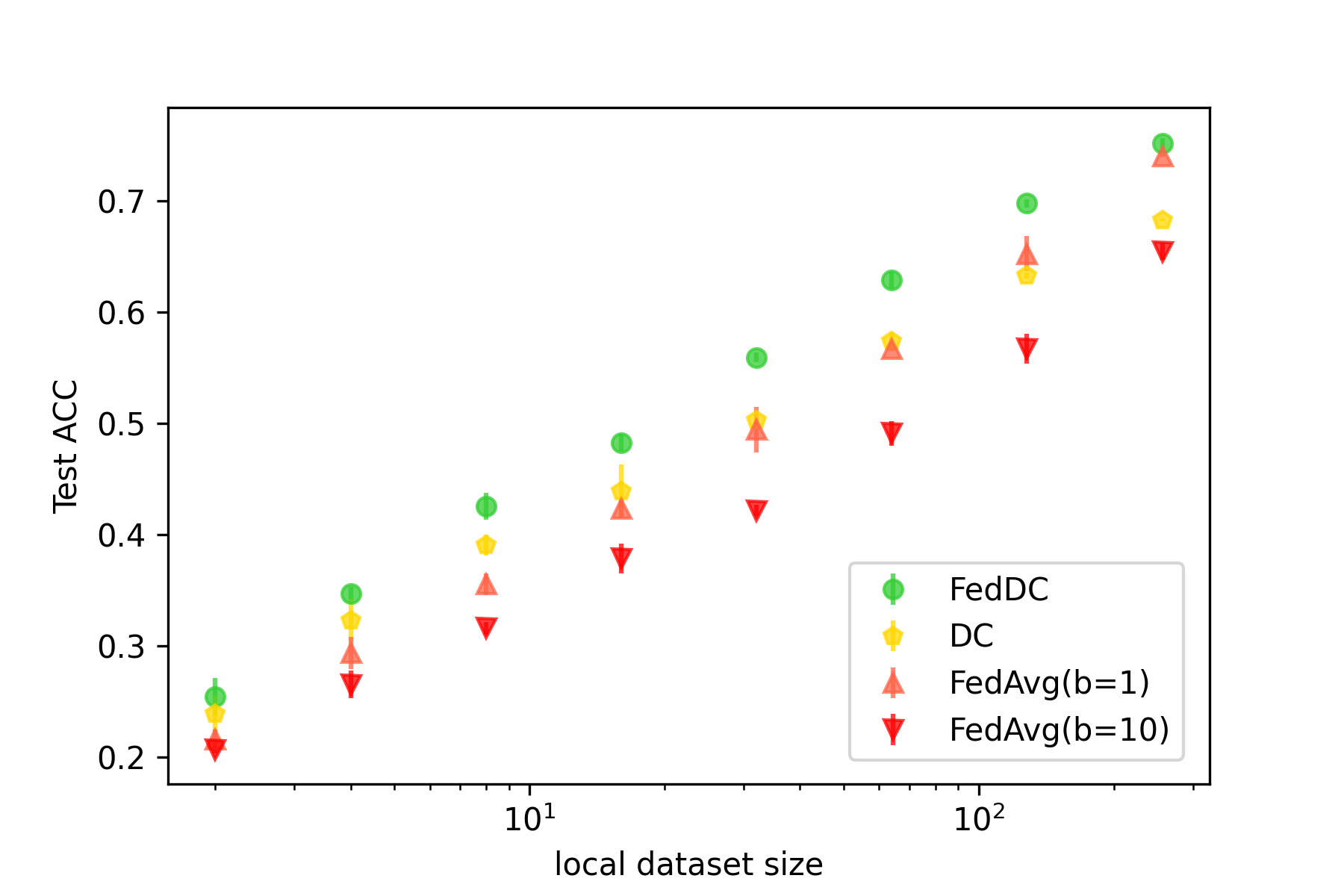}
\end{subfigure}
\caption{Test accuracy wrt. local dataset size on CIFAR10 with 150 clients ($n=2^i$ for $i\in\{1,\dots,8\}$) with linear (left) and logarithmic (right) x-axis.}
\label{fig:dataset_size}
\end{figure}

\subsubsection{Realistic Dataset Size Distribution}
\label{app:cifarothers}
In medical applications, a common scenario is that some hospitals, e.g., university clinics, hold larger datasets, while small clinics, or local doctors' offices only hold very small datasets. To simulate such a scenario, we draw local dataset sizes for a dataset of size $n$ so that a fraction $c$ of the clients hold only a minimum number of samples $n_{min}$ (the local doctor's offices), and the other clients have an increasing local dataset size starting from $n_{min}$ until all data is distributed. That is, for clients $i=\left[1,\dots,m-\lfloor cm\rfloor\right]$ the dataset sizes are given by $n_{min} + ai$ with $$a=\frac{2(n-(\lfloor cm\rfloor)n_{min})}{\lfloor(1-c)m\rfloor(\lfloor(1-c)m\rfloor - 1)}$$.

We use the MRI brainscan dataset with $c=0.3$ and $n_{min}=2$. The results presented in Tab.~\ref{tab:imbaresults} show that \ourmethod performs well in that setting. \feddc($d=4, b=10$) outperforms all other methods with an accuracy of around $0.81$, is similar to \feddc on equally distributed data ($0.79$), and is even close to the centralized gold-standard ($0.82$). 
\begin{table}[ht]
    \footnotesize
    \setlength\tabcolsep{5pt}
    \centering
    \begin{tabular}{lccc}
        \toprule
         & \textbf{MRI} \\
        \midrule
        \ourmethod (d=1, b=10) & $\mathbf{74.7}$ \scriptsize{$\mathbf{\pm 0.61}$}\\[0.05cm]
        \ourmethod (d=2, b=10) & $\mathbf{74.4}$ \scriptsize{$\mathbf{\pm 1.15}$}\\[0.05cm]
        \ourmethod (d=4, b=10) & $\mathbf{80.6}$ \scriptsize{$\mathbf{\pm 0.66}$}\\[0.05cm]
        \ourmethod (d=5, b=10) & $\mathbf{77.8}$ \scriptsize{$\mathbf{\pm 1.42}$}\\[0.05cm]
        DC (baseline) & $53.9$ \scriptsize{$\pm 0.25$}\\[0.05cm]
        \fedavg(b=1) & $70.1$ \scriptsize{$\pm 2.59$}\\[0.05cm]
        \fedavg(b=10) & $75.3$ \scriptsize{$\pm 2.37$}\\[0.05cm]
        \midrule
        {central} & $79.9$ \scriptsize{$\pm 6.23$}\\[0.05cm]
        \bottomrule
    \end{tabular}
    \vspace{0.2cm}
    \caption{Results for realistic dataset size distribution on MRI, reported is the average test accuracy of the final model over three runs ($\pm$ denotes maximum deviation from the average).}
    \label{tab:imbaresults}
\end{table}

\subsubsection{Communication efficiency of \ourmethod}
\label{app:commeff}
Although communication is not a concern in cross-silo applications, such as healthcare, the communication efficiency of \feddc is important in classical federated learning applications. We therefore compare \feddc with varying amounts of communication to \fedavg on CIFAR10 in Tab.~\ref{tab:commeff}. The results show that \feddc(d=2) outperforms \fedavg(b=1), thus outperforming just using half the amount of communication, and \feddc(d=5) performs similar to \fedavg(b=1), thus outperforming using five times less communication. \ourmethod with $d=10$ and $b=10$ significantly outperforms \fedavg(b=10), which corresponds to the same amount of communication in this low-sample setting.
\begin{table}[ht]
    \footnotesize
    \setlength\tabcolsep{5pt}
    \centering
    \begin{tabular}{lccc}
        \toprule
         & \textbf{CIFAR10} \\
        \midrule
        \ourmethod (d=1,b=10) & $\mathbf{62.9}$ \scriptsize{$\mathbf{\pm 0.02}$}\\[0.05cm]
        \ourmethod (d=2,b=10) & $\mathbf{60.8}$ \scriptsize{$\mathbf{\pm 0.65}$}\\[0.05cm]
        \ourmethod (d=5,b=10) & $\mathbf{55.4}$ \scriptsize{$\mathbf{\pm 0.11}$}\\[0.05cm]
        \ourmethod (d=10,b=20) & $\mathbf{53.8}$ \scriptsize{$\mathbf{\pm 0.47}$}\\[0.05cm]
        \fedavg(b=1) & $55.8$ \scriptsize{$\pm 0.78$}\\[0.05cm]
        \fedavg(b=10) & $48.7$ \scriptsize{$\pm 0.87$}\\[0.05cm]
        \midrule
        {central} & $65.1$ \scriptsize{$\pm 1.44$}\\[0.05cm]
        \bottomrule
    \end{tabular}
    \vspace{0.2cm}
    \caption{Communication efficiency of \ourmethod compared to \fedavg., where \feddc(d=1,b=10) and \fedavg(b=1), respectively \feddc(d=10,b=20) and \fedavg(b=10) have the same amount of communication.}
    \label{tab:commeff}
\end{table}

\subsubsection{Client Subsampling}
\label{app:clientsubsampling}
A widely used technique to improve communication-efficiency in federated learning is to subsample clients in each communication round. For example, instead of averaging the models of all clients in vanilla FedAvg, only a subset of clients sends their models and receives the average of this subset. By randomly sampling this subset, eventually all clients will participate in the process. The fraction $C\in (0,1]$ of clients sampled is a hyperparameter. In cross-SILO applications, such as healthcare, communication-efficiency is not relevant and client participation is assumed to be $C=1.0$. Client subsampling can naturally be used in \ourmethod by sampling clients both in daisy-chaining and aggregation rounds. We conducted an experiment on CIFAR10 where we compare \ourmethod using subsampling to \fedavg$(b=10)$. The results in Table~\ref{tab:clientsubsampling} show that \ourmethod indeed works well with client subsampling and outperforms \fedavg with $C=0.2$, similar to full client participation (C=1.0). However, due to the restricted flow of information, the training process is slowed. By prolonging training from $10000$ to $30000$ rounds, \ourmethod  with $C=0.2$ reaches virtually the same performance as in full client participation, but with higher variance. The same holds true for \fedavg.

\begin{table}[ht]
    \footnotesize
    \setlength\tabcolsep{5pt}
    \centering
    \begin{tabular}{llll}
        \toprule
        & \multicolumn{2}{c}{$T=10\, 000$} & $T=30\, 000$\\
         & $C=1.0$ & $C=0.2$ & $C=0.2$ \\
        \midrule
        \ourmethod (d=1,b=10) & $\mathbf{62.9}$ \scriptsize{$\mathbf{\pm 0.02}$} &  $\mathbf{53.2}$ \scriptsize{$\mathbf{\pm 3.42}$} &  $\mathbf{61.0}$ \scriptsize{$\mathbf{\pm 1.07}$}\\[0.05cm]
        \fedavg(b=10) & $48.7$ \scriptsize{$\pm 0.87$}& $45.9$ \scriptsize{$\pm 6.94$}& $49.3$ \scriptsize{$\pm 5.23$}\\[0.05cm]
        \bottomrule
    \end{tabular}
    \vspace{0.2cm}
    \caption{Client subsampling of \ourmethod compared to \fedavg on CIFAR10.}
    \label{tab:clientsubsampling}
\end{table}

\subsubsection{Additional Results on MNIST}
\label{app:mnist}
In order to further demonstrate the efficiency of \ourmethod on clients that achieve state-of-the-art performance we perform experiments on the MNIST~\citep{lecun1998gradient} dataset. We use a CNN with two convolutional layers with max-pooling, followed by two linear layers with $1024$, resp. $100$ neurons. Centralized training on all $60\,000$ training samples of MNIST achieves a test accuracy of $0.994$ which is similar to the state-of-the-art. 
The results for $m=50$ clients in Tab.~\ref{tab:mnist} show that \ourmethod outperforms \fedavg both with the same amount of communication, i.e., \fedavg($b=1$) and \fedavg($b=10$). In line with the results on CIFAR10 (cf. Fig.~\ref{fig:dataset_size}), the advantage of \ourmethod shrinks with increasing local dataset size. Using $n=1200$, i.e., the full training set distributed over $m=50$ clients, results in virtually the same performance of \ourmethod and \fedavg, both reaching a test accuracy of around $0.96$. 
\begin{table}[ht]
    \footnotesize
    \setlength\tabcolsep{5pt}
    \centering
    \begin{tabular}{lccc}
        \toprule
         & \textbf{\ourmethod(d=1,b=10)} & \textbf{\fedavg(b=1)} & \textbf{\fedavg(b=10)}\\
        \midrule
        $n=8$ & $87.6$ & $84.4$ & $84.9$ \\
        $n=1200$ & $96.7$ & $96.3$ & $96.5$\\
        \bottomrule
    \end{tabular}
    \vspace{0.2cm}
    \caption{Performance of \ourmethod and \fedavg on MNIST for varying local dataset sizes $n$.}
    \label{tab:mnist}
\end{table}

\subsection{Proof of convergence}

\begin{corollary*}
Let the empirical risks $\mathcal{E}_{emp}^i(h)=\sum_{(x,y)\in D^i}\ell(h_i(x),y)$ at each client be $L$-smooth with $\sigma^2$-bounded gradient variance and $G^2$-bounded second moments, then \ourmethod with averaging and SGD as learning algorithm has a convergence rate of $\mathcal{O}(1/\sqrt{m T})$, where $T\in\NN$ is the number of local updates.
\end{corollary*}
\begin{proof}
We assume that each client $i\in [m]$ uses SGD as learning algorithm. In each iteration $t\in [T]$, a client $i$ computes the gradient $\mathbf{G}^i_t=\nabla\ell(h^i_t(x),y)$ with $x,y\in D^i$ drawn randomly and updates the local model $h^i_{t+1} = h^i_t - \gamma G^i_t$, where $\gamma > 0$ denotes the learning rate.
\ourmethod with daisy-chaining period $d$ and averaging period $b$, run for $T$ local iterations, computes $T/b$ local gradients at each of the $m$ clients before averaging. Each local gradient is computed on an iid sample from $\Dcal$, independent of whether local models are permuted. Therefore, \ourmethod with averaging and SGD is equivalent to parallel restarted SGD (PR-SGD)~\citep{yu2019parallel} with $b/d$ times larger local datasets. \citet{yu2019parallel} analyze the convergence of PR-SGD with respect to the average $\overline{h}_t$ of local models in round $t$. Since $\mathcal{E}_{emp}^i(h)=\sum_{(x,y)\in D^i}\ell(h_i(x),y)$ at each client be $L$-smooth with $\sigma^2$-bounded gradient variance and $G^2$-bounded second moments, Theorem~1 in~\citet{yu2019parallel} is applicable. It then follows from Corollary~1~\citep{yu2019parallel} that for $\gamma = \sqrt{m}/(L\sqrt{T})$ and $b\leq T^{\frac{1}{4}}/m^{\frac{3}{4}}$
it holds that
\begin{equation*}
    \begin{split}
        \frac{1}{T}\Exp{}{\sum_{t=1}^{T}\Exp{(x,y)\sim\Dcal}{\mathcal{E}_{emp}^i\left(\overline{h}_t(x),y\right)}} \leq & \frac{2L}{\sqrt{mT}}\left(\Exp{(x,y)\sim\Dcal}{\mathcal{E}_{emp}^i\left(\overline{h}_0(x),y\right)}\right.\\
        &\left.-\Exp{(x,y)\sim\Dcal}{\mathcal{E}_{emp}^i\left(\overline{h}^*(x),y\right)}\right)\\
        &+\frac{1}{\sqrt{mT}}\left(4G^2+\sigma^2\right)\in\mathcal{O}\left(\frac{1}{\sqrt{m T}}\right)\enspace .
    \end{split}
\end{equation*}
Here, the first expectation is over the draw of local datasets and $\overline{h}^*$ is given by
\[
\overline{h}^* = \argmin_{h\in\mathcal{H}}\Exp{(x,y)\sim\Dcal}{\mathcal{E}_{emp}(\overline{h}(x),y)}.
\]
Thus, \ourmethod with averaging and SGD converges in $\mathcal{O}(1/\sqrt{m T})$.
\end{proof}


\subsection{Proof of model quality improvement by \ourmethod}
\label{app:prooffeddcworks}

In order to proof Prop.~\ref{prop:fedDCworks}, we first need the following Lemma.
\begin{lemma}
Given $\delta\in (0,1]$, $m\in\NN$ clients, and $k\in [m]$, if Algorithm~\ref{alg:feddc} with daisy chaining period $d\in\NN$ is run for $T\in\NN$ rounds with
\[
T\geq d\frac{m}{\rho^\frac1m}(H_m-H_{m-k})
\]
where $H_m$ is the $m$-th harmonic number, then each local model has seen at least $k$ distinct datasets with probability $1-\rho$.
\label{lm:minNumberRoundsDC}
\end{lemma}
%
Note that $H_m\approx \log m + \gamma +\frac{1}{2} + \mathcal{O}\left(\frac{1}{m}\right)$ where $\gamma \approx 0.5772156649$ denotes the Euler-Mascheroni-constant.

\begin{proof}
For a single local model, it follows from the coupon collector problem~\citep{neal2008generalised} that the expected number of daisy-chaining rounds $R$ required to see at least $k$ out of $m$ clients is at least $m(H_m-H_{m-k})$, where 
\[
H_m=\sum_{i=1}^m \frac{1}{i}
\]
is the $m$-th harmonic number. To see that, consider that for the first distinct client the chance to pick it is $\frac{m}{m-1}$, for the second $\frac{m}{m-2}$ and for the $k$-th it is $\frac{m}{m-k+1}$, which sums up to $m(H_m-H_{m-k})$.

Applying the Markov inequality yields 
\[
P\left(R\geq \frac{1}{\rho}m(H_m-H_{m-k})\right)\leq \rho\enspace .
\]
The probability for all local models to have seen at least $k$ clients then is at most $\rho^m$. Thus, if we perform at least 
\[
R\geq \frac{m}{\rho^\frac1m}m(H_m-H_{m-k})
\]
daisy-chaining rounds, then the probability that each local model has not seen at least $k$ distinct datasets is smaller than $\rho$. The result follows from the fact that the number of daisy-chaining rules is $R=T/d$.
\end{proof}
Note that $H_k$ can be approximated as
\[
H_m\approx \log m + \gamma +\frac{1}{2} + \mathcal{O}\left(\frac{1}{m}\right)
\]
where $\gamma=\lim_{m\rightarrow\infty}(H_m - \ln m) \approx 0.5772156649$ denotes the Euler-Mascheroni-constant. From this it follows that 
\[
H_m-H_{m-k}\approx \ln\frac{m}{m-k} + \mathcal{O}\left(\frac{1}{m}-\frac{1}{m-k}\right)
\]
With this, we can now proof Prop.~\ref{prop:fedDCworks} that we restate here for convenience.

\begin{proposition*}
Let $\modelspace$ be a model space with Radon number $r\in\NN$, $\risk$ a convex risk, and $\algo$ a learning algorithm with sample size $\samplesize_0(\epsilon,\delta)$. Given $\epsilon>0$, $\delta\in (0,r^{-1})$ and any $h\in\NN$, and local datasets $D_1,\dots,D_m$ of size $n\in\NN$ with $m\geq r^h$, then Alg.~\ref{alg:feddc} using the Radon point with aggr. period
\begin{equation*}
b\geq  d\frac{m}{\delta^\frac{1}{2m}}\left(H_m-H_{m-\left\lceil n^{-1}\samplesize_0\left(\epsilon,\sqrt{\delta}\right)\right\rceil}\right)
\end{equation*}
improves model quality in terms of $(\epsilon,\delta)$-guarantees.
\end{proposition*}
\begin{proof}
For 
\[
b\geq  d\frac{m}{\delta^\frac{1}{2m}}\left(H_m-H_{m-\left\lceil n^{-1}\samplesize_0\left(\epsilon,\sqrt{\delta}\right)\right\rceil}\right)
\]
it follows from Lemma~\ref{lm:minNumberRoundsDC} with $k=\left\lceil n^{-1}\samplesize_0\left(\epsilon,\sqrt{\delta}\right)\right\rceil$ that with probability $1-\sqrt{\delta}$ all local models are trained on at least $kn=\samplesize_0\left(\epsilon,\sqrt{\delta}\right)$ samples. Thus an $(\epsilon,\sqrt{\delta})$-guarantee holds for each model with probability $1-\sqrt{\delta}$. It follows from Eq.~\ref{eq:radonPointImprovement} that the probability that the risk is higher than $\epsilon$ is 
\[
\prob{\risk(\radonpoint_h)>\epsilon}<\left(r\sqrt{\delta}\sqrt{\delta}\right)^{2^h}=\left(r\delta\right)^{2^h}\enspace .
\]
The result follows from $\delta<r^{-1}$ and Eq.~\eqref{eq:radonPointImprovement}.
\end{proof}

\subsubsection{Numerical Analysis of Proposition~\ref{prop:fedDCworks}}
\label{app:numericalanalysisProp5}
The lower bound on the aggregation period
\[
b\geq  d\frac{m}{\delta^\frac{1}{2m}}\left(H_m-H_{m-\left\lceil n^{-1}\samplesize_0\left(\epsilon,\sqrt{\delta}\right)\right\rceil}\right)
\]
grows linearly with the daisy-chaining period and  a factor depending on the number of clients $m$, the error probability $\delta$, and the required number of hops $\left\lceil n^{-1}\samplesize_0\left(\epsilon,\sqrt{\delta}\right)\right\rceil$. Applying the bound to the experiment using the Radon point on SUSY in Sec.~\ref{sec:theory} with $m=441$ clients, daisy-chaining period $d=1$, and local dataset size of $n=2$ for local learners achieving an $(\epsilon=0.05,\delta=0.01)$-guarantee requires $b\geq 49.9$ to improve model quality according to Prop.~\ref{prop:fedDCworks}. For $b=50$ like in our experiments, Prop.~\ref{prop:fedDCworks} thus predicts that model quality is improved, under the assumption that 
\[
\samplesize_0(\epsilon,\delta) = \frac{1}{\epsilon}\log\frac{1}{\delta}\enspace .
\]
Even though, the experiments on CIFAR10 in Section~\ref{sec:experiments} are non-convex, and thus Prop.~\ref{prop:fedDCworks} does not apply, we can still evaluate the predicted required aggregation period: with $m=150$ clients, daisy-chaining period $d=1$, local dataset size of $n=64$, and local learners achieving an $(\epsilon=0.01,\delta=0.01)$-guarantee requires $b\geq 8.81$. 

We now analyze the scaling behavior with the error probability $\delta$ for various local dataset sizes in Fig.~\ref{fig:aggPeriod_deltaN}. The lower the error probability, the larger the required aggregation period $b$, in particular for small local datasets. If local datasets are sufficiently large, the aggregation period can be chosen very small. In Fig.~\ref{fig:aggPeriod_mN} we investigate the required aggregation period $b$ for local learners achieving an $(\epsilon=0.01,\delta=0.01)$-guarantee in relation to the local dataset size $n$. Indeed, the smaller the local dataset size, the larger the required aggregation period. We also see that for smaller numbers of clients, more aggregation rounds are required, since the chance of a model visiting the same client multiple times is larger.

\begin{figure}[ht]
\centering
\begin{subfigure}[t]{0.48\textwidth}
\includegraphics{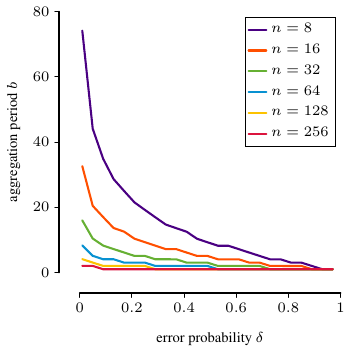}
\caption{Aggregation period $b$ for error probability $\delta$ for varying local dataset sizes with $m=150$ clients and $\epsilon=0.01$.}
\label{fig:aggPeriod_deltaN}
\end{subfigure}
\hfill
\begin{subfigure}[t]{0.48\textwidth}
\includegraphics{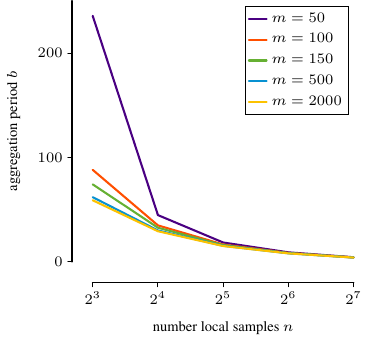}
\caption{Aggregation period $b$ for local dataset size $n$ for varying numbers of clients $m$, with $\epsilon=0.01$ and $\delta=0.01$.}
\label{fig:aggPeriod_mN}
\end{subfigure}
\end{figure}
\tikzexternalenable

\subsection{Note on Communication Complexity}
In our analysis of the communication complexity, we assume the total amount of communication to be linear in the number of communication rounds. For some communication systems the aggregated model can be broadcasted to individual clients which is not possible in daisy-chaining rounds, reducing the communication complexity in FedAvg. In most scenarios, fiber or GSM networks are used where each model has to be sent individually, so there is no substantial difference between broadcasting a model to all clients and sending an individual model to each client. Therefore, also in this case the amount of communication rounds determines the overall amount of communication.

\subsection{Extended Discussion on Privacy}
\label{app:attacks}
Federated learning only exchanges model parameters, and no local data.
It is, however, possible to infer upon local data given the model parameters, model updates or gradients~\citep{ma2020safeguarding}. In classical federated learning there are two types of attacks that would allow such inference: (i) an attacker intercepting the communication of a client with the server, obtaining models and model updates to infer upon the clients data, and (ii) a malicious server obtaining models to infer upon the data of each client. A malicious client cannot learn about a specific other client's data, since it only obtains the average of all local models. In federated daisy-chaining there is a third possible attack: (iii) a malicious client obtaining model updates from another client to infer upon its data. 

In the following, we discuss potential defenses against these three types of attacks in more detail. 
Note that we limit the discussion on attacks that aim at inferring upon local data, thus breaching data privacy. 
Poisoning~\citep{bhagoji2019analyzing}  and backdoor~\citep{sun2019can} attacks are an additional threat in federated learning, but are of less importance for our main setting in healthcare: there is no obvious incentive for a hospital to poison a prediction. It is possible that \ourmethod presents a novel risk surface for those attacks, but such attack strategies are non-obvious. Robust aggregation, such as the Radon point, are suitable defenses against such attacks~\citep{liu_threats_2022}. Moreover, the standard mechanisms that guarantee differential privacy also defend against backdoor and poising attacks~\citep{sun2019can}.

A general and wide-spread approach to tackle all three possible attack types is to add noise to the model parameters before sending. Using appropriate clipping and noise, this guarantees $\epsilon,\delta$-differential privacy for local data~\citep{wei2020federated} at the cost of a slight-to-moderate loss in model quality. We empirically demonstrated that \ourmethod performs well under such noise in Sec.~\ref{sec:privacy}.

Another approach to tackle an attack on communication (i) is to use encrypted communication. One can furthermore protect against a malicious server (ii) by using homomorphic encryption that allows the server to average models without decrypting them~\citep{zhang2020batchcrypt}. This, however, only works for particular aggregation operators and does not allow to perform daisy-chaining. Secure daisy-chaining in the presence of a malicious server (ii) can, however, be performed using asymmetric encryption. Assume each client creates a public-private key pair and shares the public key with the server. To avoid the malicious server to send clients its own public key and act as a man in the middle, public keys have to be announced (e.g., by broadcast). While this allows sending clients to identify the recipient of their model, no receiving client can identify the sender. Thus, inference on the origin of a model remains impossible. For a daisy-chaining round the server sends the public key of the receiving client to the sending client, the sending client checks the validity of the key and sends an encrypted model to the server which forwards it to the receiving client. Since only the receiving client can decrypt the model, the communication is secure. 

In standard federated learning, a malicious client cannot infer specifically upon the data of another client from model updates, since it only receives the aggregate of all local models. In federated daisy-chaining, it receives the model from a random, unknown client in each daisy-chaining round. Now, the malicious client can infer upon the membership of a particular data point in the local dataset of the client the model originated from, i.e., through a membership inference attack~\citep{shokri2017membership}. Similarly, the malicious client can infer upon the presence of data points with certain attributes in the dataset~\citep{ateniese2015hacking}. 
The malicious client, however, does not know the client the model was trained on, i.e., it does not know the origin of the dataset. 
Using a random scheduling of daisy-chaining and aggregation rounds at the server, the malicious client cannot even distinguish between a model from another client or the average of all models. Nonetheless, daisy-chaining opens up new potential attack vectors (e.g., clustering received models to potentially determine their origins). These potential attack vectors can be tackled in the same way as in standard federated learning, i.e., by adding noise to model parameters as discussed above, since ``[d]ifferentially private models are, by construction, secure against membership inference attacks''~\citep{shokri2017membership}.

\end{document}